\documentclass{article}

\usepackage{arxiv}

\usepackage{graphicx}
\usepackage{amssymb}
\usepackage[table]{xcolor}

\usepackage[utf8]{inputenc} 
\usepackage[T1]{fontenc}    
\usepackage{hyperref}       
\usepackage{url}            
\usepackage{booktabs}       
\usepackage{amsfonts}       
\usepackage{nicefrac}       
\usepackage{microtype}      
\usepackage{lipsum}		

\newtheorem{thm}{Theorem}[section]
\newtheorem{theorem}[thm]{Theorem}
\newtheorem{definition}[thm]{Definition}
\newtheorem{example}[thm]{Example}
\newtheorem{proposition}[thm]{Proposition}
\newtheorem{lemma}[thm]{Lemma}
\newtheorem{proof}{Proof}

\title{Logic Negation with Spiking Neural P Systems}

\author{Daniel Rodr\'{\i}guez-Chavarr\'{\i}a  \\
Dept. of Computer Science and Artificial Intelligence \\
University of Seville, Spain\\
\texttt{danrodcha@gmail.com}\\
\And
Miguel A. Guti\'errez-Naranjo \\
Dept. of Computer Science and Artificial Intelligence \\
University of Seville, Spain\\
\texttt{magutier@us.es}
\And
Joaqu\'{\i}n Borrego-D\'{\i}az \\
Dept. of Computer Science and Artificial Intelligence \\
University of Seville, Spain\\
\texttt{jborrego@us.es}
}

\begin{document}
\maketitle

\begin{abstract}
Nowadays, the success of neural networks as reasoning systems is doubtless. Nonetheless, one of the drawbacks of such reasoning systems is that they work as black-boxes and the acquired knowledge is not human readable. In this paper, we present a new step in order to close the gap between connectionist and logic based reasoning systems. We show that two of the most used inference rules for obtaining negative information in rule based reasoning systems, the so-called {\it Closed World Assumption} and {\it Negation as Finite Failure} can be characterized by means of {\it spiking neural P systems}, a formal model of the third generation of neural networks born in the framework of membrane computing.
\end{abstract}


\section{Introduction}
Neural networks are nowadays one of the most promising tools in computer sciences. They have been successfully applied to many real-world domains and the number of application fields is continuously increasing \cite{Goodfellow-et-al-2016}. Beyond this doubtless success, one of the main drawbacks of such systems is that they work as black-boxes, i.e., the learned knowledge through the training process is not human-readable. Learning process in neural networks consists basically of optimizing parameters (usually a huge amount of them) guided by some type of gradient-based method and the resulting model is usually far from having {\it semantic} sense for a human researcher. In fact, the problem of {\it explainability} is becoming a new research frontier in artificial intelligence systems, even beyond machine learning \cite{DBLP:journals/access/AdadiB18,DBLP:journals/corr/abs-1806-00069}. Due to this lack of readability, new studies about the integration of neural network models (the so-called {\it connectionist} systems) and logic-based systems \cite{BaderHitzler2005,Besold201597,DBLP:series/sci/2007-77,DBLP:journals/pr/MontavonLBSM17,DBLP:journals/tnn/SamekBMLM17,DBLP:journals/tnn/TranG18} can shed a new light on the future development of both research areas\footnote{A recent survey in neural-symbolic learning and reasoning can be found in \cite{DBLP:journals/corr/abs-1711-03902}.}.

In this context, the computational framework known as {\it spiking neural P systems} \cite{IbarraLPW2010,DBLP:journals/fuin/IonescuPY06} (SN P systems, for short) provides a formal framework for the integration of both disciplines: on the one hand, they use {\it spikes} (electrical impulses) as discrete units of information as in logic-based methods and, on the other hand, their models consist of graphs where the information flows among nodes as in standard neural network architectures. SN P systems belong to the third generation of neural network models \cite{DBLP:journals/nn/Maass97}, the so-called {\it integrate-and-fire} spiking neuron models \cite{gerstner2002spiking}. The integration of logic and neural networks via spikes takes advantage from an important biological fact: all the spikes inside a biological brain look alike. By using this feature, a computational binary code can be considered: sending one spike is considered as a sign for {\it true} and if no spikes are sent, then it is considered as a sign of {\it false}. These features were exploited in \cite{DBLP:journals/ijon/Diaz-PernilG18} where SN P systems were used to bridge bioinspired connectionist systems with the semantics of reasoning systems based on logic. 

The main contribution of this paper is to add new elements for dealing with negation in the interplay of bridging neural networks and logic. Bridges between both areas can help to enrich each other. 


In this study, we go on with the approach started in \cite{DBLP:journals/ijon/Diaz-PernilG18} by focusing on logic negation. Using negation in computational logic systems is often a hard task \cite{AptB1994} since pure derivative reasoning systems have no way to derive negative information from a set of facts and rules. This problem is solved by adding to the reasoning system a new inference rule which allow to {\it derive} negative information. In this paper, two of such inference rules are studied in the framework of SN P systems: {\it Closed World Assumption (CWA)} and {\it Negation as Finite Failure}\footnote{A detailed description of the controversy generated around the use of the negation in deductive databases is out of the scope of this paper. More information is available in \cite{DBLP:conf/adbt/Clark77,DBLP:conf/ijcai/JaffarLL83}.}.

Loosely speaking, given a deductive database $KB$, CWA considers {\it false} all the atomic sentences which are not logical consequence of $KB$. The attempts to check whether a sentence is a logical consequence of $KB$ or not can fall into an infinite loop, and therefore, a different {\it effective} rule is needed. Such rule is Negation as Finite Failure. It  considers {\it false} a sentence if all the attempts to prove it fail (according to some protocol). This is a quite restrictive definition of negation, but it is on the basis on many reasoning systems used in Artificial Intelligence as the Logic Programming paradigms \cite{DBLP:conf/ecai/Kakas92} or planning systems \cite{DBLP:conf/iclp/Lifschitz99}. 


The recent development of SN P systems involves SN P systems with communication on request \cite{DBLP:journals/ijns/PanPZN17}, applications of fuzzy SN P systems \cite{DBLP:conf/bic-ta/HuangZWRHW16,WangZZHWP2015}, {\it cell-like} SN P systems \cite{DBLP:journals/tcs/WuZPP16}, SN P systems {\it with request rules} \cite{DBLP:journals/ijon/SongP16}, SN P systems {\it with structural plasticity} \cite{DBLP:journals/nca/CabarleAPS15}, SN P systems {\it with thresholds} \cite{DBLP:journals/neco/ZengZSP14} or SN P systems with rules on synapses \cite{DBLP:journals/ijon/SongZLZ15} among many others.

The paper is organized as follows: Section \ref{Pre} recalls some basics on SN P systems and the procedural and declarative semantics of deductive databases. The following section shows how the inference rules CWA and Negation as Finite Failure can be characterized via SN P systems. Finally, some conclusions are showed in Section \ref{Con}. 

\section{Preliminaries}\label{Pre}
In this section, we recall briefly some basic concepts on SN P systems and the declarative and procedural semantics of deductive databases.

\subsection{Spiking Neural P Systems}
SN P systems were introduced as a model of computational devices inspired by the flow of information between neurons. This model keeps the basic idea of encoding and processing the information via binary events used other spiking neuron models (see, e.g., Ch. 3 in \cite{Floreano:2008:BAI:1457317}. Such devices are distributed and work in a parallel way. They consist of a directed graph with {\it neurons} placed on the nodes. 
Each neuron contains a number of copies of an object called the {\it spike} and it may contain several \emph{firing} and \emph{forgetting} rules. Firing rules send {\it spikes} to other neurons. Forgetting rules allow to remove spikes from a neuron. In order to decide if a rule is applicable, the contents of the neuron is checked against a regular set associated with the rule. In each time unit, if several rules can be applied in a neuron, one of them, non-deterministically chosen, must be used. In this way, rules are used in a sequential way in each neuron, but neurons function in parallel with each other. As usual, a global clock with discrete time steps is assumed and the functioning of the whole system is synchronized.

Formally, an SN P system of the degree $m\ge 1$ is a construct\footnote{In the literature, many different SN P systems models have been presented. In this paper, a simple model is considered.}

\[
   \Pi=(O,\sigma_1,\sigma_2,\ldots,\sigma_m,syn)
\]
where $O=\{a\}$ is the singleton alphabet ($a$ is called \emph{spike}) and $\sigma_1,\sigma_2,\ldots,\sigma_m$ are \emph{neurons}. Each neuron is a pair $\sigma_i = (n_i,R_i)$, $1\le i\le m$, where:
         \begin{enumerate}
            \item $n_i \ge 0$ is the \emph{initial number of spikes} contained
              in $\sigma_i$;
            \item $R_i$ is a finite set of \emph{rules} of the following two kinds:
                  \begin{itemize}
                     \item[(1)] \emph{firing} rules of type
                   $E/a^p \to a^q$, where $E$ is a regular
                   expression over the spike $a$ and $p,q \ge 1$ are integer numbers
                   ;
                     \item[(2)] \emph{forgetting} rules of type $a^s \to \lambda$, with $s$ an integer number such that
                   $s \ge 1$;
                  \end{itemize}
         \end{enumerate}
The set of synapses (edges) $syn$ is a set of pairs $syn \subseteq \{1,2,\ldots,m\} \times \{1,2,\ldots,m\}$, verifying that $(i,i)$ does not belong to $sys$ for any $i\in \{1,\dots,m\}$. 

Let us suppose that the neuron $\sigma_i$ contains $k$ spikes and a rule $E/a^p \to a^q$  with $k\ge p$. Let $L(E)$ be the language generated by the regular expression $E$. In these conditions, if $a^k$ belongs to $L(E)$, then the rule $E/a^p \to a^q$ can be applied. The application is performed by sending $q$ spikes to all neurons $\sigma_j$ such that $(i,j) \in syn$ and deleting $p$ spikes from $\sigma_i$ (thus only $k-p$ spikes remain into the neuron). In this case, it is said that the neuron is {\it fired}.  

Let us now suppose that the neuron $\sigma_i$ contains exactly $s$ spikes and the forgetting rule $a^s \to \lambda$. In such case, the rule can be fired by removing all the $s$ spikes from $\sigma_i$. If the regular expression $E$ in a firing rule $E/a^p \to a^q$  is equal to $a^p$, then the firing rule can be expressed as $a^p \to a^q$. In each time unit, if a neuron $\sigma_i$ can use one of its rules, then one of them must be used. If two or more rules can be applied in a neuron, then only one of them is non-deterministically chosen regardless of its type. The SN P system evolves according to these type of rules and reaches different configurations which are represented as vectors ${\mathbb C}_j=(t_1^j,\dots,t_m^j)$ where $t_k^j$ stands for the number of spikes at the neuron $\sigma_k$ in the $j-th$ configuration. It will be useful to consider only the first components of a configuration. Let us define ${\mathbb C}_j[1,\dots,n] =(t_1^j,\dots,t_n^j)$ as the $n$-dimensional vector composed by the $n$ first components of ${\mathbb C}_j$. The initial configuration is the vector with the number of spikes in each neuron at the beginning of the computation ${\mathbb C}_0=(n_1, n_2, \dots , n_m)$. By using the rules described above, transitions between configurations can be defined. A sequence of transitions which starts at the initial configuration is called a computation. 

\subsection{Declarative Semantics of Rule-based Deductive Databases}
Reasoning based on rules can be formalized according to different approaches. In this paper, propositional logic is considered for representing knowledge. Different formal representation systems where the number of available {\it terms} is finite (as those based on pairs attribute-value of first order logic representations without function symbols) can be bijectively mapped onto propositional logic systems and therefore the presented approach covers many real-life cases.

Next, some basics on propositional logic are provided. Let $\{p_1,\dots, p_n\}$ be a set of variables. A literal is a variable or a negated variable. An expression $L_1\wedge \dots \wedge L_n \to A$, where $n\ge 0$, $A$ is a variable and $L_1,\dots L_n$ are literals is a rule. The conjunction $L_1\wedge\dots\wedge L_n$ is the {\it body} and the variable $A$ is the {\it head} of the rule. If $n=0$, the body of the rule is empty. A finite set of rules $KB$ is called a deductive database. A mapping $I: \{p_1,\dots, p_n\} \to \{0,1\}$ is an {\it interpretation}, which is usually represented as a vector $(i_1,\dots, i_n)$ with $I(p_k)=i_k\in\{0,1\}$ for $k\in\{1,\dots, n\}$. The interpretations $I_\downarrow$ and $I_\uparrow$ are defined as $I_\downarrow = (0,\dots,0)$ and $I_\uparrow = (1,\dots,1)$. The set of all the interpretations on a set of $n$ variables denote by $2^n$. Given two interpretations $I_1$ and $I_2$, $I_1 \subseteq I_2$ if for all $k\in\{1,\dots, n\}$, $I_1(p_k)=1$ implies $I_2(p_k)=1$; $I_1 \cup I_2$  and $I_1\cap I_2$ are new interpretations such that $(I_1\cup I_2)(p_k) = max\{I_1(p_k),I_2(p_k)\}$ and $(I_1\cap I_2)(p_k) = min\{I_1(p_k),I_2(p_k)\}$ for $k\in\{1,\dots ,n\}$. An operator $S:2^n\to 2^n$ is {\it monotone} if for all interpretations $I_1$ and $I_2$, if $I_1\subseteq I_2$, then $S(I_1)\subseteq S(I_2)$. An interpretation $I$ is extended in the following way: $I(\neg p_i) = 1 - I(p_i)$ for a variable $p_i$; $I(L_1\wedge\dots\wedge L_n) = \min\{I(L_1),\dots,I(L_n)\}$ and for a rule\footnote{Let us remark that, according to the definition, $I(\to A) = 1$ if and only if $I(A) = 1$.}

$$I(L_1\wedge \dots \wedge L_n \to A) =
\left\{
\begin{array}{ll}
0 & \mbox{if } I(L_1\wedge \dots \wedge L_n) = 1\mbox{ and }I(A)=0 \\
1 & \mbox{otherwise}
\end{array}
\right.$$

Next, we recall the notions of {\it model} and {\it F-model} of a deductive database. The concept of $F$-model is one of the key ideas in this paper. To the best of our knowledge, it was firstly presented in \cite{Yang1991}. The definition used in this paper is adapted from the original one.

\begin{definition}
Let $I$ be an interpretation for a deductive database $KB$ 
\begin{itemize}
\item $I$ is a {\it model} if for all rule  $L_1\wedge \dots \wedge L_n \to A$ verifying that $\displaystyle \min_{i\in\{1,\dots, n\}} I(L_i)=1$, the equality $I(A)=1$ holds; in other words, if $I(R)=1$ for all rule $R \in KB$.
\item $I$ is a {\it F-model} if for all rule  $L_1\wedge \dots \wedge L_n \to A$ verifying that $I(A)=1$, the equality $\displaystyle \max_{i\in\{1,\dots, n\}} I(L_i)=1$ holds.
\end{itemize}
\end{definition}

Next example illustrates these concepts.

\begin{example}\label{Ex_1}
Let $KB$ be the deductive database on the set $\{p_1,p_2,p_3,p_4,p_5,p_6,$ $p_7,p_8,p_9\}$ defined as follows:

$$\begin{array}{rrrrrrrrrlr}
R_1 \equiv & \rightarrow p_1  & \hspace{3mm}  & R_4 \equiv & p_3 \wedge p_6 \rightarrow p_4  & \hspace{3mm} & R_7 \equiv  &    p_6 \rightarrow p_5 & \hspace{3mm} & R_9 \equiv   &    p_7 \rightarrow p_2 \\
R_2 \equiv & p_1 \rightarrow p_2  & & R_5 \equiv &    p_4            \rightarrow p_5  & & R_8 \equiv   &    p_8 \rightarrow p_9  & & R_{10} \equiv &    p_9 \rightarrow p_8 \\
R_3 \equiv & p_1 \wedge p_2 \rightarrow p_3  & & R_6 \equiv &    p_7 \rightarrow p_6 & & & &  & & \\
\end{array}$$


\noindent then, the interpretation represented by the vector $I_1=(1,1,1,0,0,0,0,0,0)$ is a model of $KB$.
In this case, the rules verifying 
$$\displaystyle \min_{i\in\{1,\dots, n\}} I(L_i)=1$$ 
are $R_1$, $R_2$ and $R_3$ and all of them satisfies $I(A)=1$.

The interpretations $I_2 = (0,0,0,1,1,1,1,0,0)$ and $I_3=(0,0,0,1,1,1,1,1,1)$ are {\it F-models} of $KB$. In both cases,
for all rule  $L_1\wedge \dots \wedge L_n \to A$ verifying that $I(A)=1$, the equality $$\displaystyle \max_{i\in\{1,\dots, n\}} I(L_i)=1$$ holds. If we consider the interpretation $I_2$, then $I_2(p_4) = I_2(p_5) =I_2(p_6) =I_2(p_7) = 1$ and all the rules with variables $p_4$, $p_5$, $p_6$ or $p_7$ in their heads ($R_4$, $R_5$ and $R_6$) verify that there exists a variable $q$ in the body of the rule with $I(q)=1$. The case of $I_3$ is analogous.

\end{example}

In a certain sense, {\it F-models} keep a {\it duality} with respect the concept of models. 
If $I_A$ and $I_B$ are models, then $I_A\cap I_B$ is also a model and if they are {\it F-models}, then $I_A \cup I_B$ is also a {\it F-model} \cite{Yang1991}. Next, the definition of {\it Failure Operator} $F_{KB}$ of a deductive database $FK$ is recalled. It can be seen as a dual of the Kowalski's immediate consequence operator $T_{KB}$ \cite{DBLP:journals/jacm/EmdenK76}. 

\begin{definition}
Let $\{p_1,\dots,p_n\}$ be a set of variables and $KB$ a deductive database on it. The failure operator of $KB$ is the mapping $F_{KB}: 2^n \to 2^n$ such that for $I \in 2^n$, $F_{KB}(I)$ is an interpretation $F_{KB}(I): \{p_1,\dots,p_n\} \to \{0,1\}$ where $F_{KB}(I)(p_k)=1$ if for each rule $L_1\wedge \dots \wedge L_n \to p_k$ in $KB$, $\max_{i\in\{1,\dots, n\}} I(L_i)=1$ holds (for $k\in\{1,\dots,n\}$); otherwise, $F_{KB}(I)(p_k)=0$. 
\end{definition}

Let us remark that, according to the definition, if there is no rule in $KB$ with $p_k$ in its head, then $F_{KB}(I)(p_k)=1$, for all interpretation $I$.

\vspace{2mm}

Kowalski's operator $T_{KB}$ allows to characterize the models of $KB$ (see, e.g. \cite{HitzlerSeda2011}) in the sense that an interpretation $I$ is a model of $KB$ if and only if $T_{KB}(I) \subseteq I$. Proposition \ref{Pr_1} shows that the failure operator $F_{KB}$ also allows to characterize the {\it F-models}. The intuition behind the failure operator is to capture the idea of {\it immediate failure} in a similar way that the operator $T_{KB}$ captures the idea of {\it immediate consequence}.

\begin{proposition}\label{Pr_1} \cite{Yang1991} Let $KB$ be a deductive database.
\begin{itemize}
\item An interpretation $I_F$ is an $F$-model of $KB$ if and only if $I_F\subset F_{KB}(I)$.
\item The failure operator $F_{KB}$ is monotone, over the set of the interpretations of $KB$.
\end{itemize}
\end{proposition}

Since the image of an interpretation by the $F_{KB}$ operator is an interpretation itself, it can be iteratively applied.

\begin{definition}
Let $KB$ be a deductive database and $F_{KB}$ its failure operator. 
\begin{itemize}
\item[(a)] The mapping $F_{KB}\downarrow: {\mathbb N} \to 2^n$ is defined as follows: $F_{KB}\downarrow 0 = I_\downarrow$ and $F_{KB}\downarrow n = F_{KB}\,(F_{KB}\downarrow (n-1))$ if $n>0$. In the limit, it is also considered $$\displaystyle F_{KB}\downarrow \omega = \bigcup_{k\ge 0}F_{KB}\downarrow k$$

\item[(b)]The mapping $F_{KB}\uparrow: {\mathbb N} \to 2^n$ is defined as follows: $F_{KB}\uparrow 0 = I_\uparrow$ and $F_{KB}\uparrow n = F_{KB}\,(F_{KB}\uparrow (n-1))$ if $n>0$. In the limit, it is also considered $$\displaystyle F_{KB}\uparrow \omega = \bigcap_{k\ge 0}F_{KB}\uparrow k$$
\end{itemize}
\end{definition}

Bearing in mind that the number of rules and variables in a deductive database are finite, the next Proposition is immediate.
\begin{proposition}\label{Cor_1}
Let $KB$ be a deductive database and $F_{KB}$ its failure operator.
\begin{itemize}
\item[(a)] There exists $n \in \mathbb{N}$ such that $F_{KB}\uparrow n = F_{KB}\uparrow k$ for all $k \geq n$.
\item[(b)] There exists $n \in \mathbb{N}$ such that $F_{KB}\downarrow n = F_{KB}\downarrow k$ for all $k \geq n$.
\end{itemize}
\end{proposition}

Let us remark that Prop. \ref{Cor_1} implies that the number of computation steps for reaching the above limits is finite.

\begin{example}\label{Ex_2}
Let us consider again the database $KB$ used in Example \ref{Ex_1} and its failure operator. The following interpretations are obtained.

$$\begin{array}{l}
F_{KB}\downarrow 0 = I_\downarrow = (0,0,0,0,0,0,0,0,0)\\
F_{KB}\downarrow 1 =  F_{KB} (F_{KB}\downarrow 0) = (0,0,0,0,0,0,1,0,0)\\
F_{KB}\downarrow 2 =  F_{KB} (F_{KB}\downarrow 1) = (0,0,0,0,0,1,1,0,0)\\
F_{KB}\downarrow 3 =  F_{KB} (F_{KB}\downarrow 2) = (0,0,0,1,0,1,1,0,0)\\
F_{KB}\downarrow 4 =  F_{KB} (F_{KB}\downarrow 3) = (0,0,0,1,1,1,1,0,0)
\end{array}$$

Since $F_{KB}\downarrow 5 = F_{KB}\downarrow 4$, then $F_{KB}\downarrow \omega = (0,0,0,1,1,1,1,0,0)$

$$\begin{array}{l}
F_{KB}\uparrow 0 = I_\uparrow = (1,1,1,1,1,1,1,1,1)\\
F_{KB}\uparrow 1 =  F_{KB} (F_{KB}\uparrow 0) = (0,1,1,1,1,1,1,1,1)\\
F_{KB}\uparrow 2 =  F_{KB} (F_{KB}\uparrow 1) = (0,0,1,1,1,1,1,1,1)\\
F_{KB}\uparrow 3 =  F_{KB} (F_{KB}\uparrow 2) = (0,0,0,1,1,1,1,1,1)
\end{array}$$
Since $F_{KB}\uparrow 4 = F_{KB}\uparrow 3$, then $F_{KB}\uparrow \omega = (0,0,0,1,1,1,1,1,1)$

\end{example}

\subsection{Procedural Semantics of Rule-based Deductive Databases}
The main result of this paper is the characterization of the set of variables obtained by the non-monotonic inference rules CWA  and {\it Negation as Finite Failure} via the  procedural behaviour of an SN P system. For the sake of completeness, some basics of the procedural semantics of deductive databases are recalled\footnote{A detailed description can be found in \cite{lloyd1987foundations}.}. A {\it goal} is a formula $\neg B_1 \vee \dots \vee \neg B_n$ where $B_i$ are atoms. As usual, the goal $\neg B_1 \vee \dots \vee \neg B_n$ will be represented as $B_1,\dots,B_n \rightarrow$. We also consider the empty clause $\Box$ as a goal. Given a goal $G \,\equiv\, A_1,\dots, A_{k-1},A_k,A_{k+1},\dots, A_n \rightarrow$ and a rule $R\, \equiv\, B_1,\dots, B_m \rightarrow A_k$, the goal
$$G' \equiv A_1,\dots, A_{k-1},B_1,\dots, B_m,A_{k+1},\dots, A_n \rightarrow$$
is called the {\it resolvent} of $R$ and $G$. It is also said that $G'$ is derived from $R$ and $G$. Let $KB$ be a deductive database and $G$ a goal. An SLD-derivation of $KB \cup \{G\}$ consists of a (finite or infinite) sequence $G_0,G_1,\dots$ of goals with $G_0=G$ and a sequence of rules $R_1,R_2,\dots$ from $KB$ such that $G_{i+1}$ is derived from $R_{i+1}$ and $G_i$. 
It is said that $KB \cup \{G\}$ has a finite failed tree if all the SLD-derivations are finite and none of them  has the empty clause $\Box$ as the last goal of the derivation. The {\it failure set} of $KB$ is the set of all variables $A$ for which there exists a finite failed tree for $KB \cup \{A \rightarrow \}$.

\begin{example}\label{Ex_3}
Let $KB$ be the same deductive database from Example \ref{Ex_1}. Next, some SLD derivations are calculated:
$$
\begin{array}{ccccccccc}
KB \cup  \{ p_3 \rightarrow \} & & & KB \cup  \{ p_9 \rightarrow \} & & & KB \cup  \{ p_6 \rightarrow \}\\
\begin{array}{ll}
Rule \; used & Goals \\
& p_3 \rightarrow  \\
R_3 & p_1,p_2 \rightarrow  \\
R_2 & p_1 \rightarrow  \\
R_1 & \Box \\\\
\end{array} & & &
\begin{array}{ll}
Rule \; used & Goals \\
& p_9 \rightarrow  \\
R_8 & p_8 \rightarrow \\
R_{10} & p_9 \rightarrow  \\
R_8 & p_8 \rightarrow \\
\vdots & \vdots
\end{array} & & &
\begin{array}{ll}
Rule \; used & Goals \\
& p_6 \rightarrow  \\
R_6 & p_7 \rightarrow \\\\\\\\
\end{array}
\end{array}$$
As shown above, the goals $p_3 \rightarrow$ and $p_9 \rightarrow$ do not have finite failed trees whereas $p_6 \rightarrow$ does.  Finally, it is easy to check that the {\it failure set} of $KB$ is $\{ p_4,p_5,p_6,p_7 \}$.
\end{example}

We give now a brief recall of the formal definition of both inference rules. A detailed motivation of such rules is out of the scope of this paper. The first inference rule
for deriving such negative information considered in this paper is the CWA \cite{DBLP:conf/adbt/Reiter77}: {\it If A is not a logical consequence of $KB$, then infer $\neg A$}. 
The second inference rule called {\it Negation as Finite Failure} \cite{DBLP:conf/adbt/Clark77}: {\it If $KB \cup \{A \rightarrow\}$ has a finite failed tree, then infer $\neg A$}, or, in other words, if $A$ belongs to the failure set, then infer $\neg A$.

The next Theorem is an adaptation of the Th.13.6 in \cite{lloyd1987foundations} and provides a procedural characterization of the variables in the failure set of a deductive database $KB$. It settles the equality of two sets defined with two different approaches: on the one hand, the set of variables such that all the SLD-derivations fail after a finite number of steps and, on the other hand, the set of variables mapped onto 1 by the interpretation $F_{KB}\downarrow \omega$, obtained by the iteration of the failure operator.

\begin{theorem}\label{Th_1}
Let $KB$ be a database on a set of variables $\{p_1,\dots,p_n\}$ and $F_{KB}$ its failure operator. For all $k$ in $\{1,\dots,n\}$, $p_k$ is in the failure set of $KB$ if and only if $F_{KB}\downarrow \omega (p_k)=1$
\end{theorem}

By using this theorem, we will prove in the next section that the finite failure set of a database $KB$ can be characterized by means of SN P systems. The next Theorem relates the CWA with the failure operator. A proof of it is out of the scope of this paper. Details can be found in  \cite{lloyd1987foundations} and \cite{Yang1991}.

\begin{theorem}\label{Th_2}
Let $KB$ be a database on a set of variables $\{p_1,\dots,p_n\}$ and $F_{KB}$ its failure operator. For all $k$ in $\{1,\dots,n\}$, $p_k$ is not a logical consequence of $KB$ if and only if $F_{KB}\uparrow \omega (p_k)=1$
\end{theorem}

\section{Logic Negation with SN P Systems}\label{Neg}
In this section, we bridge the neural model of SN P systems with the inference rules CWA and {\it Negation as Finite Failure}. The main theorems in this paper claim that the result of both inference rules can be computed in a finite number of steps by an appropriate SN P system. The proof of such results is achieved via some lemmas which link the properties of the SN P systems with the semantics of the deductive databases.

\begin{theorem}\label{Th_3}
Let us consider a set of variables $\{p_1,\dots, p_n\}$ and a deductive database $KB$ on it. Let $I$ be an interpretation on such set of variables. Let $F_{KB}$ be the failure operator of $KB$. An SN P system can be constructed from $KB$ such that $$F_{KB}(I)= {\mathbb C}_3[1,\dots,n]$$ \noindent where ${\mathbb C}_3$ is the configuration of the SN P system after the third step of computation.
\end{theorem}

Theorem \ref{Th_3} claims the equality of two $n$-dimensional vectors. The first one is the vector which represents the interpretation $F_{KB}(I): \{p_1,\dots, p_n\} \to \{0,1\}$ obtained by means of the application of the operator $F_{KB}$ to the interpretation $I$. The second one is the vector which represents the number of spikes in the neurons $\sigma_1,\dots,\sigma_n$ in the corresponding SN P system in the third configuration. The proof is constructive and it builds explicitly the SN P system. 

\begin{proof}
Let $KB$ be a deductive database such that $\{ r_1 , \dots r_k \}$ and $\{ p_1 , \dots p_n \}$  are the set of rules and the set of variables. Given a variable $p_i$, the number of rules which have $p_i$ in the head is denoted by $h_i$ and given a rule $r_j$, the number of variables in its body is denoted by $b_j$. The SN P system of degree $2n+k+2$.
$$\Pi_{KB} =(O,\sigma_1,\sigma_2,\ldots,\sigma_{2n+k+2},syn)$$ can be constructed as follows:
\begin{itemize}
\item $O=\{a\}$;
\item $\sigma_j = (0,\{a\to \lambda\})$ for $j\in \{1, \dots n\}$
\item $\sigma_{n+j} = (i_j,R_j)$, $j\in \{1, \dots n\}$, where 
\begin{itemize}
\item[$\bullet$] $i_j=I(p_j)$ if $h_j = 0$
\item[$\bullet$] $i_j=I(p_j)\cdot h_j$ if $h_j > 0$
\end{itemize}
and $R_j$ is the set of $h_j$ rules
\begin{itemize}
\item[$\bullet$] $R_j=  \{a \longrightarrow a \}$ if $h_j=0 $
\item[$\bullet$] $R_j=  \{a^{h_j} \longrightarrow a \}$  if $ h_j > 0 $
\end{itemize}

\item $\sigma_{2n+j} = (0,R_j)$, $j\in \{1, \dots k\}$, where $R_j$ is one of the following set of rules
\begin{itemize}
\item[$\bullet$] $R_j = \emptyset$ if $b_j=0$.
\item[$\bullet$] $R_j =  \{ a^l \to a \,|\, l \in\{1,\dots, b_j\}\,\}$ if $b_j>0$
\end{itemize}
\end{itemize}
For the sake of simplicity, the neurons $\sigma_{2n+k+1}$ and $\sigma_{2n+k+2}$ will be denoted by $\sigma_G$ and $\sigma_T$, respectively.
\begin{itemize}
\item $\sigma_G = (1, \{a\to a\})$
\item $\sigma_T = (0, \{a\to a\})$
\item $syn =$ \phantom{$\cup\,$} $\{(n+i,i)\,|\,i\in\{1,\dots,n\}\}$\newline
\phantom{$syn =$} $\cup\,$
$\left\{
\begin{array}{ll}
(n+i,2n+j)\,|\, & i\in\{1,\dots,n\}, j\in\{1,\dots,k\} \\
                & \mbox{and $p_i$ is a variable in the body of $r_j$}
\end{array}
\right\}$\newline
\phantom{$syn =$} $\cup\,$
$\left\{
\begin{array}{ll}
(2n+j,n+i)\,|\, & i\in\{1,\dots,n\}, j\in\{1,\dots,k\} \\
                & \mbox{and $p_i$ is the variable in the head of $r_j$}
\end{array}
\right\}$\newline
\phantom{$syn =$} $\cup\,$
$\{(G,T),(T,G)\}$\newline
\phantom{$syn =$} $\cup\,$
$\left\{
\begin{array}{ll}
(T,n+i)\,|\, & i\in\{1,\dots,n\} \\
              & \mbox{and $p_i$ is a variable such that $h_i = 0$}
\end{array}
\right\}$
\end{itemize}

The proof will be split into four lemmas. Although the result of the theorem only concerns to the third configuration, the lemmas are proved in general.

Before going on with the proof, the building of the SN P system is illustrated with the following toy example.

\begin{example}\label{Ex_new}
Let us consider the set of three variables $\{p_1,p_2,p_3\}$, a database on it with two rules
$$ r_1 \equiv p_1 \to p_2 \hspace{3cm} r_1 \equiv p_1,p_2 \to p_3$$
and the interpretation $I_{\downarrow}=(0,0,0)$. According to the notation, in this case $n=3$, $k=2$, $h_1=0$, $h_2=1$, $h_3=1$, $b_1=1$ and $b_2=2$. The associated SN P system has $2n+k+2=10$ neurons and its initial configuration is depicted in Fig. \ref{Fig_new}. Since the interpretation is $I_{\downarrow}$, there is only one spike in this first configuration ${\mathbb C}_0$. It is placed on the neuron $\sigma_G$. The rule $a\to a$ in $\sigma_G$ is applied and the unique spike in the configuration ${\mathbb C}_1$ is placed in $\sigma_T$. Since $\sigma_T$ has two outcoming synapses, the application of the rule $a\to a$ in it produces two spikes. Therefore, in the configuration ${\mathbb C}_2$ there are two spikes in the SN P system: one of them in $\sigma_G$ and the other one in $\sigma_4$. The application of the rule $a\to a$ in $\sigma_G$ sends one spike to $\sigma_T$, so in this neuron there a spike in the configuration ${\mathbb C}_3$. Since $\sigma_4$ has two outcoming synapses, the application of the rule $a\to a$ sends one spike to $\sigma_1$ and another to $\sigma_8$. To sum up, in the confiruration ${\mathbb C}_3$, there are three spikes in the system, each of them in the neurons $\sigma_T$, $\sigma_1$ and $\sigma_8$. According to the theorem, in order to know $F_{KB}(I_{\downarrow})$ it suffices to check the number of spikes in the neurons $\sigma_1$, $\sigma_2$ and $\sigma_3$ at the configuraton ${\mathbb C}_3$. In other words, this SN P system has computed $F_{KB}(I_{\downarrow}) = (1,0,0)$.

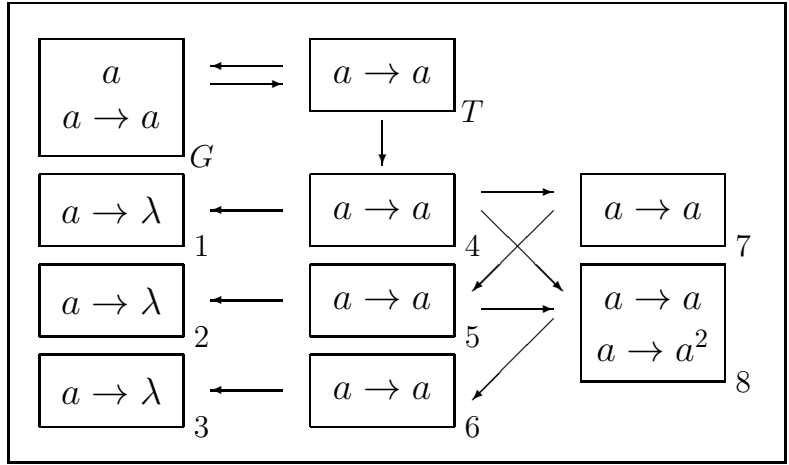
\begin{figure}[t]
\begin{center}
\fbox{
\setlength{\unitlength}{12mm}
\begin{picture}(8.2,4.9)

\put(0.1,3.3){\line(1,0){1.6}}
\put(0.1,4.6){\line(1,0){1.6}}
\put(0.1,3.3){\line(0,1){1.3}}
\put(1.7,3.3){\line(0,1){1.3}}
\put(0.9,4.2){\makebox(0,0){\Large $a$}}
\put(0.9,3.7){\makebox(0,0){\Large $a \to a$}}
\put(1.9,3.3){\makebox(0,0){\large $G$}}

\put(0.1,2.3){\line(1,0){1.6}}
\put(0.1,3.1){\line(1,0){1.6}}
\put(0.1,2.3){\line(0,1){0.8}}
\put(1.7,2.3){\line(0,1){0.8}}
\put(0.9,2.7){\makebox(0,0){\Large $a \to \lambda$}}
\put(1.9,2.3){\makebox(0,0){\large $1$}}

\put(0.1,1.3){\line(1,0){1.6}}
\put(0.1,2.1){\line(1,0){1.6}}
\put(0.1,1.3){\line(0,1){0.8}}
\put(1.7,1.3){\line(0,1){0.8}}
\put(0.9,1.7){\makebox(0,0){\Large $a \to \lambda$}}
\put(1.9,1.3){\makebox(0,0){\large $2$}}

\put(0.1,0.3){\line(1,0){1.6}}
\put(0.1,1.1){\line(1,0){1.6}}
\put(0.1,0.3){\line(0,1){0.8}}
\put(1.7,0.3){\line(0,1){0.8}}
\put(0.9,0.7){\makebox(0,0){\Large $a \to \lambda$}}
\put(1.9,0.3){\makebox(0,0){\large $3$}}

\put(3.1,3.8){\line(1,0){1.6}}
\put(3.1,4.6){\line(1,0){1.6}}
\put(3.1,3.8){\line(0,1){0.8}}
\put(4.7,3.8){\line(0,1){0.8}}
\put(3.9,4.2){\makebox(0,0){\Large $a \to a$}}
\put(4.9,3.8){\makebox(0,0){\large $T$}}

\put(3.1,2.3){\line(1,0){1.6}}
\put(3.1,3.1){\line(1,0){1.6}}
\put(3.1,2.3){\line(0,1){0.8}}
\put(4.7,2.3){\line(0,1){0.8}}
\put(3.9,2.7){\makebox(0,0){\Large $a \to a$}}
\put(4.9,2.3){\makebox(0,0){\large $4$}}

\put(3.1,1.3){\line(1,0){1.6}}
\put(3.1,2.1){\line(1,0){1.6}}
\put(3.1,1.3){\line(0,1){0.8}}
\put(4.7,1.3){\line(0,1){0.8}}
\put(3.9,1.7){\makebox(0,0){\Large $a \to a$}}
\put(4.9,1.3){\makebox(0,0){\large $5$}}

\put(3.1,0.3){\line(1,0){1.6}}
\put(3.1,1.1){\line(1,0){1.6}}
\put(3.1,0.3){\line(0,1){0.8}}
\put(4.7,0.3){\line(0,1){0.8}}
\put(3.9,0.7){\makebox(0,0){\Large $a \to a$}}
\put(4.9,0.3){\makebox(0,0){\large $6$}}

\put(6.1,2.3){\line(1,0){1.6}}
\put(6.1,3.1){\line(1,0){1.6}}
\put(6.1,2.3){\line(0,1){0.8}}
\put(7.7,2.3){\line(0,1){0.8}}
\put(6.9,2.7){\makebox(0,0){\Large $a \to a$}}
\put(7.9,2.3){\makebox(0,0){\large $7$}}

\put(6.1,0.8){\line(1,0){1.6}}
\put(6.1,2.1){\line(1,0){1.6}}
\put(6.1,0.8){\line(0,1){1.3}}
\put(7.7,0.8){\line(0,1){1.3}}
\put(6.9,1.7){\makebox(0,0){\Large $a\to a$}}
\put(6.9,1.2){\makebox(0,0){\Large $a \to a^2$}}
\put(7.9,0.8){\makebox(0,0){\large $8$}}

\put(2.8,0.7){\vector(-1,0){0.8}}
\put(2.8,1.7){\vector(-1,0){0.8}}
\put(2.8,2.7){\vector(-1,0){0.8}}
\put(2.8,4.3){\vector(-1,0){0.8}}
\put(2.0,4.1){\vector(1,0){0.8}}
\put(3.9,3.7){\vector(0,-1){0.5}}
\put(5.0,2.9){\vector(1,0){0.8}}
\put(5.0,2.7){\vector(1,-1){0.9}}
\put(5.8,2.7){\vector(-1,-1){0.9}}
\put(5.0,1.6){\vector(1,0){0.8}}
\put(5.8,1.5){\vector(-1,-1){0.9}}
\end{picture}
}
\end{center}

\caption{Initial configuration of the SN P system from Example \ref{Ex_new}. \label{Fig_new}}
\end{figure}
\end{example}

Next, the following lemmas will be proved.

\begin{lemma}
For all $t\ge 0$, in the $2t$-th configuration ${\mathbb C}_{2t}$ the neuron $\sigma_{G}$ has exactly one spike and  $\sigma_{T}$ is empty.
\end{lemma}

\begin{proof}
The result will be proved by induction. The lemma holds in the initial configuration. The inductive assumption is that in the configuration ${\mathbb C}_{2t}$, the neuron $\sigma_{T}$  does not contain spikes and the neuron $\sigma_{G}$ contains exactly one spike. There is only one incoming synapse in $\sigma_{G}$ which comes from $\sigma_{T}$, and vice versa. Furthermore, the unique rule that occurs in each neuron is $a\to a$ so, in ${\mathbb C}_{2t+1}$, $\sigma_{G}$ has consumed its spike and does not contain any spike, and the neuron $\sigma_{T}$ contains exactly one spike. For the same reasoning, in ${\mathbb C}_{2t+2}$, $\sigma_{T}$ has consumed its spike and the neuron $\sigma_{G}$ contains exactly one spike.
\end{proof}

\begin{lemma}
For all $t\ge 0$ the following results hold:
\begin{itemize}
\item  For all $p\in\{1,\dots, k\}$ the neuron $\sigma_{2n+p}$ is empty in the configuration ${\mathbb C}_{2t}$.
\item For all $q\in\{1,\dots, n\}$, the neuron $\sigma_{n+q}$ is empty in the configuration ${\mathbb C}_{2t+1}$.
\end{itemize}
\end{lemma}

\begin{proof}
In the configuration ${\mathbb C}_0$, for all $p\in\{1,\dots, k\}$, the neuron $\sigma_{2n+p}$ is empty and, for all $q\in\{1,\dots, n\}$, each neuron $\sigma_{n+q}$ contains, at most, $h_q$ spikes. These spikes are consumed by the application of the rule $a^{h_j} \to a$ (or $a\to a$). Finally, as every neuron with synapse to $\sigma_{n+q}$ is empty at ${\mathbb C}_0$, it follows that in the configuration ${\mathbb C}_1$, all the neurons $\sigma_{n+q}$ with $q\in\{1,\dots, n\}$ are empty.

As induction hypothesis, we state that in ${\mathbb C}_{2t}$, for all $p\in\{1,\dots, k\}$, the neuron $\sigma_{2n+p}$ is empty and for all $q\in\{1,\dots, n\}$, the neuron $\sigma_{n+q}$ is empty in the configuration ${\mathbb C}_{2t+1}$. As defined before, the number of incoming synapses in each neuron $\sigma_{j}$ is $b_j$. The neurons which are the origin of such synapses send (at most) one spike in one computational step, so in ${\mathbb C}_{2t+1}$, the number of spikes in the neuron $\sigma_{2n+p}$ is, at most, $b_{p}$. The corresponding rules ($a^{h_j} \to a$ or $a\to a$) consume all these spikes so, at ${\mathbb C}_{2t+1}$, all the neurons with outgoing synapses to $\sigma_{2n+p}$ are empty. In the next step, at most, $b_{p}$ spikes contained in the neurons $\sigma_{2n+p}$ were consumed by the corresponding rules. Therefore,  we conclude that at ${\mathbb C}_{2t+2}$, for all $p\in\{1,\dots, k\}$, the neurons $\sigma_{2n+p}$ are empty. 

Focusing on the second part of the lemma, as induction hypothesis we state that the neurons $\sigma_{n+q}$ with $q\in\{1,\dots, n\}$ are empty in the configuration ${\mathbb C}_{2t+1}$. Each neuron $\sigma_{n+q}$ can receive at most $h_q$ if $h_q>1$ and 1 if $h_q=0$, since there are $h_q$ or 1 incoming synapses and each of these sends, at most, one spike. Hence, at  ${\mathbb C}_{2t+2}$, $\sigma_{n+q}$ has, at most, $h_q$ if $h_q>1$ and 1 if $h_q=0$, spikes. All of them are consumed by the corresponding rule and, since all the neurons which can send spikes to $\sigma_{n+q}$ are empty at ${\mathbb C}_{2t+2}$, we conclude that, for all $q\in\{1,\dots, n\}$, the neuron $\sigma_{n+q}$ is empty in the configuration ${\mathbb C}_{2t+3}$.
\end{proof}

\begin{lemma}
For all $q\in\{1,\dots, n\}$, the neuron $\sigma_{q}$ is empty in the configuration ${\mathbb C}_{2t}$.
\end{lemma}

\begin{proof}
In the first configuration (${\mathbb C}_0$) the lemma holds. For ${\mathbb C}_{2t}$ with $t>0$ it is enough to check that, as stated in Lemma 2, for all $q\in\{1,\dots, n\}$, the neuron $\sigma_{n+q}$ is empty in the configuration ${\mathbb C}_{2t+1}$ and each $\sigma_q$ receives at most one spike in each computation step from the corresponding $\sigma_{n+q}$. Therefore, in each configuration ${\mathbb C}_{2t+1}$, each neuron $\sigma_q$ contains, at most, one spike. Since such spike is consumed by the rule $a\to \lambda$ and no new spike arrives, then the neuron $\sigma_{q}$ is empty in the configuration ${\mathbb C}_{2t}$.
\end{proof}


\begin{lemma}
Let $I=(i_1,\dots,i_n)$ be an interpretation for $KB$ and let $S=(s_1,\dots,s_n)$ be a vector with the following properties. For all $j\in\{1,\dots, n\}$
\begin{itemize}
\item If $i_j=0$ and $h_j=0$, then $s_j=0$ 
\item If $i_j=0$ and $h_j>0$, then $s_j\in\{0,\dots , h_j-1\}$.
\item If $i_j\not=0$ and $h_j=0$, then $s_j=1$
\item If $i_j\not=0$ and $h_j>0$, then $s_j=h_j$
\end{itemize}
If in the configuration ${\mathbb C}_{2t}$ the neuron $\sigma_{n+j}$ contains exactly $s_j$ spikes for all $j\in\{1,\dots, n\}$ then the interpretation obtained by applying the failure operator $F_{KB}$ to the interpretation $I$, $F_{KB}(I)$, is $(q_1,\dots,q_n)$ where $q_j$, $j\in\{1,\dots,n\}$, corresponds with the number of spikes contained in the neuron $\sigma_j$ in the configuration ${\mathbb C}_{2t+3}$.
\end{lemma}

\begin{proof}
Let us consider $m\in\{1,\dots,n\}$ and $F_{KB}(I)(p_m)=1$. We will prove that in the configuration ${\mathbb C}_{2t+3}$ there is exactly one spike in the neuron $\sigma_m$.

If $F_{KB}(I)(p_m)=1$, then for each rule $r_l\equiv L_{d_1} \wedge \dots \wedge L_{d_l} \to p_m$ with $p_m$ in the head, there exists $j\in\{1,\dots,n\}$ such that $I(L_j)=1$.

{\bf Case 1:} Let us consider that there is no such rule $r_l$. By construction, the neuron $\sigma_{n+m}$ has only one incoming synapse from neuron $\sigma_T$; and according to the previous lemmas:
\begin{itemize}
\item In ${\mathbb C}_{2t}$ the neuron $\sigma_G$ contains exactly one spike.
\item For all $q\in\{1,\dots, n\}$, the neuron $\sigma_{n+q}$ is empty in the configuration ${\mathbb C}_{2t+1}$
\item For all $q\in\{1,\dots, n\}$, the neuron $\sigma_{q}$ is empty in the configuration ${\mathbb C}_{2t}$.
\end{itemize}

In these conditions, the corresponding rules in $\sigma_G$ and $\sigma_{n+m}$ are fired and in ${\mathbb C}_{2t+1}$, the neuron $\sigma_{T}$ contains one spike. In ${\mathbb C}_{2t+2}$, the neuron $\sigma_{n+m}$ contains one spike and $\sigma_m$ is empty. Finally, in the next step $\sigma_{n+m}$ sends one spike to $\sigma_m$, so, in ${\mathbb C}_{2t+3}$, $\sigma_m$ contains one spike.

{\bf Case 2:} Let us now consider that there are $h_m$ rules (with $h_m > 0$) such that $r_l\equiv L_{d_1}\wedge \dots L_{d_l}\to p_m$ and for each one there exists $j_l\in\{1,\dots,n\}$ such that $I(L_{j_l})=1$. This means that, in ${\mathbb C}_{2t}$, every neuron $\sigma_{n+j_l}$ contains 1 or $h_{j_l}$ spikes, as appropriate. All these neurons fire the corresponding rules, and, in ${\mathbb C}_{2t+1}$, every $\sigma_{2n+l}$ has at least one spike. So one rule from $\{ a^q \to a \,|\, q \in\{1,\dots, b_{l}\}\,\}$ is fired in every $\sigma_{2n+l}$ and in ${\mathbb C}_{2t+2}$ the neuron $\sigma_{n+m}$ contains exactly $h_m$ spikes. The corresponding rule fires and the neuron $\sigma_m$ contains one spike in ${\mathbb C}_{2t+3}$.

\end{proof}

Finally, the proof of the Th. \ref{Th_3} is provided. It is immediate from Lemma 4.

\vspace{2mm}

{\it Proof.} Let us note that one of the possible vectors $S=(s_1,\dots,s_n)$ obtained from the interpretation $I$ is exactly the same interpretation $I=(i_1,\dots,i_n)$. If we also consider the case when $t=0$, we have proved that from the initial configuration ${\mathbb C}_0$ where $i_k$ and $h_k$ indicates the number of spikes in the neuron $\sigma_{n+k}$, then the configuration ${\mathbb C}_3$ encodes $F_{KB}(I)$.

\end{proof}

Theorem \ref{Th_3} is the basis of the two main results of this paper, which are proved in the following theorems.

\begin{theorem}\label{Th_5}
Let $KB$ be a deductive database on the set of variables $\{p_1,\dots, p_k\}$. An SN P system can be constructed from $KB$ such that it computes the inference rule $CWA$ on the database $KB$. 
\end{theorem}

\begin{proof}

According to Th. \ref{Th_2}, $\neg p_k$ is inferred from $KB$ by using the inference rule $CWA$ if and only is $F_{KB}\uparrow \omega (p_k)=1$ and from Th. \ref{Th_3}, an SN P system can be constructed from $KB$ such that $F_{KB}(I)= {\mathbb C}_3[1,\dots,n]$ where ${\mathbb C}_3$ is the configuration of the SN P system after the third step of computation. By combining both results, we will prove

$$(\forall z \ge 1)\, F_{KB}\uparrow z = {\mathbb C}_{2z+1}[1,\dots,n]$$ 

\noindent where ${\mathbb C}_{2z+1}[1,\dots,n]$ is the vector whose components are the spikes on the neurons $\sigma_1,\dots, \sigma_n$ in the configuration  $ {\mathbb C}_{2z+1}$. We will prove it by induction.

For $z=1$, we will see that $F_{KB}\uparrow 1 = F_{KB}(F_{KB}\uparrow 0)=F_{KB}(I_\uparrow)$ is the vector whose components are the spikes on the neurons $\sigma_1,\dots, \sigma_n$ in the configuration  ${\mathbb C}_3$. The result holds from {\it Lemma 4} in the proof of Th. \ref{Th_3}. By induction, let us consider now that $F_{KB}\uparrow z = {\mathbb C}_{2z+1}[1,\dots,n]$ holds. As previously stated, this means that in the previous configuration ${\mathbb C}_{2z}$ the spikes in the neurons $\sigma_{n+1},\dots, \sigma_{2n}$ can be represented as a vector $S=(s_1,\dots,s_n)$ with the properties claimed in {\it Lemma 4}, namely, if the neuron $\sigma_j$ has no spikes in ${\mathbb C}_{2z+1}$, then $s_j=0$ or $s_j\in\{0,\dots, h_j-1\}$, as corresponds, and, if the neuron $\sigma_j$ has one spike in ${\mathbb C}_{2z+1}$, then $s_j=1$ or $s_j=h_j$, as appropriate. Hence, according to {\it Lemma 4}, three computational steps after ${\mathbb C}_{2z}$,  $F_{KB}({\mathbb C}_{2z+1}[1,\dots,n])$ is computed: $$F_{KB}\uparrow z+1 = F_{KB}(F_{KB}\uparrow z)= F_{KB}({\mathbb C}_{2z+1}[1,\dots,n])={\mathbb C}_{2z+3}[1,\dots,n]$$

From corollary \ref{Cor_1}, there exists $m \in \mathbb{N}$ such that $F_{KB}\uparrow m = F_{KB}\uparrow k$ for all $k \geq m$, i.e. $F_{KB}\uparrow m = F_{KB}\uparrow \omega$. So the vector whose components are the spikes on the neurons $\sigma_1 , ... , \sigma_n$ in the configuration $\mathbb{C}_{2m+1}$ is the result obtained by applying the inference rule $CWA$. \hfill $\Box$ 


\end{proof}

The previous proof can be adapted to prove that the SN P systems also can  characterize the inference rule {\it Negation of Failure Set}.

\begin{theorem}\label{Th_4}
Let $KB$ be a deductive database on the set of variables $\{p_1,\dots, p_k\}$. An SN P system can be constructed from $KB$ such that it computes the inference rule {\it Negation of Failure Set} on the database $KB$. 
\end{theorem}

\begin{proof}
According to Th. \ref{Th_1}, $\neg p_k$ is inferred from $KB$ by using the inference rule {\it Negation of Failure Set} if and only is $F_{KB}\downarrow \omega (p_k)=1$ and from Th. \ref{Th_3}, an SN P system can be constructed from $KB$ such that $F_{KB}(I)= {\mathbb C}_3[1,\dots,n]$ where ${\mathbb C}_3$ is the configuration of the SN P system after the third step of computation. By combining both results, we will prove

$$(\forall z \ge 1)\, F_{KB}\downarrow z = {\mathbb C}_{2z+1}[1,\dots,n]$$ 

\noindent where ${\mathbb C}_{2z+1}[1,\dots,n]$ is the vector whose components are the spikes on the neurons $\sigma_1,\dots, \sigma_n$ in the configuration  $ {\mathbb C}_{2z+1}$. We will prove it by induction.

For $z=1$, we have to prove that $F_{KB}\downarrow 1 = F_{KB}(F_{KB}\downarrow 0)=F_{KB}(I_\downarrow)$ is the vector whose components are the spikes on the neurons $\sigma_1,\dots, \sigma_n$ in the configuration  ${\mathbb C}_3$. The result holds from {\it Lemma 4} in the proof of Th. \ref{Th_3}. By induction, let us consider now that $F_{KB}\downarrow z = {\mathbb C}_{2z+1}[1,\dots,n]$ holds. As previously stated, this means that in the previous configuration ${\mathbb C}_{2z}$ the spikes in the neurons $\sigma_{n+1},\dots, \sigma_{2n}$ can be represented as a vector $S=(s_1,\dots,s_n)$ with the properties claimed in {\it Lemma 4}, namely, if the neuron $\sigma_j$ has no spikes in ${\mathbb C}_{2z+1}$, then $s_j=0$ or $s_j\in\{0,\dots, h_j-1\}$, as corresponds, and, if the neuron $\sigma_j$ has one spike in ${\mathbb C}_{2z+1}$, then $s_j=1$ or $s_j=h_j$, as appropriate. Hence, according to {\it Lemma 4}, three computation steps after ${\mathbb C}_{2z}$,  $F_{KB}({\mathbb C}_{2z+1}[1,\dots,n])$ is computed: $$F_{KB}\downarrow z+1 = F_{KB}(F_{KB}\downarrow z)= F_{KB}({\mathbb C}_{2z+1}[1,\dots,n])={\mathbb C}_{2z+3}[1,\dots,n]$$

From corollary \ref{Cor_1}, there exists $m \in \mathbb{N}$ such that $F_{KB}\downarrow m = F_{KB}\downarrow k$ for all $k \geq m$, i.e. $F_{KB}\downarrow m = F_{KB}\downarrow \omega$. So the vector whose components are the spikes on the neurons $\sigma_1 , ... , \sigma_n$ in the configuration $\mathbb{C}_{2m+1}$ is the result obtained by applying the inference rule {\it Negation of Failure Set}. \hfill $\Box$ 

\end{proof}

\begin{example}\label{Ex_????}
Let us consider the deductive database $KB$ from Example \ref{Ex_1} and the SN P system associated to $KB$.  Its graphical representation is shown in Fig. \ref{snps_1}.

\begin{figure}[t]
\begin{center}
\includegraphics[scale=0.3]{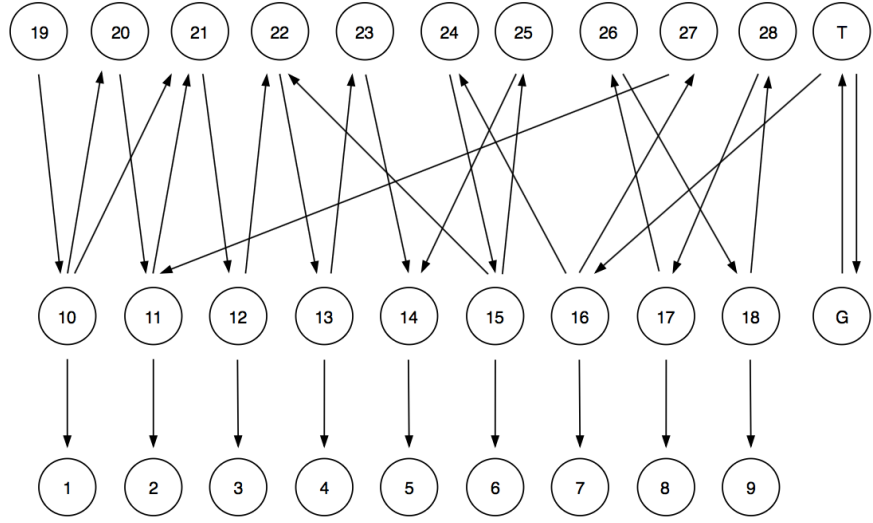}
\caption{Graphical representation of the synapses of the SN P system obtained from Example \ref{Ex_1}.}\label{snps_1}
\end{center}
\end{figure}
\end{example}

All steps of the computation (downwards and upwards) are shown in Table \ref{Tab_1}. Note that in Table \ref{Tab_1} the solution of applying failure operator every step is codified on the neurons $\sigma_1 , ... , \sigma_n$ (\textit{grey cells}).

\setlength{\tabcolsep}{1.5pt}
{ \renewcommand{\arraystretch}{0.8}
\begin{table}
\begin{tabular}{ccc}
\begin{tabular}{|c|c|c|c|c|c|c|c|c|c|c|}
\hline
 & $\mathbb{C}_0$ & $\mathbb{C}_1$ &$\mathbb{C}_2 $& $\mathbb{C}_3 $&$ \mathbb{C}_4$ &$ \mathbb{C}_5 $&$ \mathbb{C}_ 6$ & $\mathbb{C}_ 7 $&$ \mathbb{C}_8 $&$ \mathbb{C}_9 $\\\hline
$\sigma_1   $  & 0 & \cellcolor{gray!25}0 & 0 & \cellcolor{gray!25}0 & 0 &\cellcolor{gray!25} 0 & 0 & \cellcolor{gray!25}0& 0  & \cellcolor{gray!25}0 \\\hline
$\sigma_2 $    & 0 & \cellcolor{gray!25}0 & 0 & \cellcolor{gray!25}0 & 0 &\cellcolor{gray!25} 0 & 0 & \cellcolor{gray!25}0& 0  & \cellcolor{gray!25}0 \\\hline
$\sigma_3 $   & 0 & \cellcolor{gray!25}0 & 0 & \cellcolor{gray!25}0 & 0 &\cellcolor{gray!25} 0 & 0  & \cellcolor{gray!25}0& 0  & \cellcolor{gray!25}0 \\\hline
$\sigma_4  $   & 0 & \cellcolor{gray!25}0 & 0 & \cellcolor{gray!25}0 & 0 & \cellcolor{gray!25}0 & 0  & \cellcolor{gray!25}1& 0  & \cellcolor{gray!25}1 \\\hline
$\sigma_5  $   & 0 & \cellcolor{gray!25}0 & 0 &\cellcolor{gray!25} 0 & 0 &\cellcolor{gray!25} 0 & 0  & \cellcolor{gray!25}0& 0  & \cellcolor{gray!25}1 \\\hline
$\sigma_6  $   & 0 & \cellcolor{gray!25}0 & 0 & \cellcolor{gray!25}0 & 0 & \cellcolor{gray!25}1 & 0  & \cellcolor{gray!25}1& 0  & \cellcolor{gray!25}1 \\\hline
$\sigma_7 $    & 0 & \cellcolor{gray!25}0 & 0 & \cellcolor{gray!25}1 & 0 & \cellcolor{gray!25}1 & 0  &\cellcolor{gray!25}1& 0  & \cellcolor{gray!25}1 \\\hline
$\sigma_8  $   & 0 & \cellcolor{gray!25}0 & 0 & \cellcolor{gray!25}0 & 0 & \cellcolor{gray!25}0 & 0  & \cellcolor{gray!25}0& 0  & \cellcolor{gray!25}0 \\\hline
$\sigma_9   $  & 0 & \cellcolor{gray!25}0 & 0 & \cellcolor{gray!25}0 & 0 & \cellcolor{gray!25}0 & 0  & \cellcolor{gray!25}0 & 0  & \cellcolor{gray!25}0 \\\hline
$\sigma_{10}$ & 0 & 0 & 0 & 0 & 0 & 0 & 0 & 0& 0& 0 \\\hline
$\sigma_{11}$ & 0 & 0 & 0 & 0 & 1 & 0 & 1 & 0& 1 & 0\\\hline
$\sigma_{12}$ & 0 & 0 & 0 & 0 & 0 & 0  & 0 & 0& 0& 0 \\\hline
$\sigma_{13}$ & 0 & 0 & 0 & 0 & 0 & 0  & 1 & 0 &1 &0\\\hline
$\sigma_{14}$ & 0 & 0 & 0 & 0 & 0 & 0 & 1 & 0 & 2 & 0\\\hline
$\sigma_{15}$ & 0 & 0 & 0 & 0 & 1 & 0  & 1 & 0& 1 & 0 \\\hline
$\sigma_{16}$ & 0 & 0 & 1 & 0 & 1 & 0  & 1 & 0& 1 & 0 \\\hline
$\sigma_{17}$ & 0 & 0 & 0 & 0 & 0 & 0  & 0 & 0 & 0 & 0 \\\hline
$\sigma_{18}$ & 0 & 0 & 0 & 0 & 0 & 0  & 0 & 0& 0 & 0 \\\hline
$\sigma_{19}$ & 0 & 0 & 0 & 0 & 0& 0 & 0 & 0& 0 & 0  \\\hline
$\sigma_{20} $& 0 & 0 & 0 & 0 & 0& 0 & 0 & 0& 0& 0 \\\hline
$\sigma_{21} $& 0 & 0 & 0 & 0 & 0 & 0 & 0 &  0& 0& 0\\\hline
$\sigma_{22} $& 0 & 0 & 0 & 0 & 0 & 1 & 0 & 1 & 0&1\\\hline
$\sigma_{23} $& 0 & 0 & 0 & 0 & 0 & 0 & 0 &  1& 0&1\\\hline
$\sigma_{24} $& 0 & 0 & 0 & 1 & 0 & 1 & 0 &  1& 0&1\\\hline
$\sigma_{25} $& 0 & 0 & 0 & 0 & 0 & 1 & 0 &  1& 0&1\\\hline
$\sigma_{26}$ & 0 & 0 & 0 & 0 & 0 & 0 & 0 & 0& 0& 0\\\hline
$\sigma_{27} $& 0 & 0 & 0 & 1 & 0 & 1 & 0 &  1& 0&1\\\hline
$\sigma_{28} $& 0 & 0 & 0 & 0 & 0 & 0 & 0 &  0& 0& 0\\\hline
$\sigma_{T}  $ & 0 & 1 & 0 & 1 & 0 & 1 & 0 & 1 & 0& 1\\\hline
$\sigma_{G}$  & 1 & 0  & 1 & 0 & 1 & 0 & 1 & 0& 1& 0 \\\hline
\end{tabular}
& \hspace{1cm} &
\begin{tabular}{|c|c|c|c|c|c|c|c|c|}
\hline
 & $\mathbb{C}_0$ & $\mathbb{C}_1$ & $\mathbb{C}_2$ & $\mathbb{C}_3$ & $\mathbb{C}_4$ & $\mathbb{C}_5$ & $\mathbb{C}_ 6$ & $\mathbb{C}_ 7$ \\\hline
$\sigma_1$     & 0 & \cellcolor{gray!25}1 & 0 & \cellcolor{gray!25}0 & 0 &\cellcolor{gray!25} 0 & 0 & \cellcolor{gray!25}0 \\\hline
$\sigma_2$     & 0 & \cellcolor{gray!25}1 & 0 & \cellcolor{gray!25}1 & 0 &\cellcolor{gray!25} 0 & 0 & \cellcolor{gray!25}0 \\\hline
$\sigma_3$     & 0 & \cellcolor{gray!25}1 & 0 & \cellcolor{gray!25}1 & 0 &\cellcolor{gray!25} 1 & 0  & \cellcolor{gray!25}0 \\\hline
$\sigma_4   $  & 0 & \cellcolor{gray!25}1 & 0 & \cellcolor{gray!25}1 & 0 & \cellcolor{gray!25}1 & 0  & \cellcolor{gray!25}1 \\\hline
$\sigma_5   $  & 0 & \cellcolor{gray!25}1 & 0 &\cellcolor{gray!25} 1 & 0 &\cellcolor{gray!25} 1 & 0  & \cellcolor{gray!25}1 \\\hline
$\sigma_6   $  & 0 & \cellcolor{gray!25}1 & 0 & \cellcolor{gray!25}1 & 0 & \cellcolor{gray!25}1 & 0  & \cellcolor{gray!25}1 \\\hline
$\sigma_7  $   & 0 & \cellcolor{gray!25}1 & 0 & \cellcolor{gray!25}1 & 0 & \cellcolor{gray!25}1 & 0  &\cellcolor{gray!25}1 \\\hline
$\sigma_8  $   & 0 & \cellcolor{gray!25}1 & 0 & \cellcolor{gray!25}1 & 0 & \cellcolor{gray!25}1 & 0  & \cellcolor{gray!25}1 \\\hline
$\sigma_9  $   & 0 & \cellcolor{gray!25}1 & 0 & \cellcolor{gray!25}1 & 0 & \cellcolor{gray!25}1 & 0  & \cellcolor{gray!25}1 \\\hline
$\sigma_{10} $& 1 & 0 & 0 & 0 & 0 & 0 & 0 & 0 \\\hline
$\sigma_{11} $& 2 & 0 & 2 & 0 & 0 & 0 & 0 & 0\\\hline
$\sigma_{12}$ & 1 & 0 & 1 & 0 & 1 & 0  & 0 & 0 \\\hline
$\sigma_{13} $& 1 & 0 & 1 & 0 & 1 & 0  & 1 & 0 \\\hline
$\sigma_{14} $& 2 & 0 & 2 & 0 & 2 & 0 & 2 & 0 \\\hline
$\sigma_{15} $& 1 & 0 & 1 & 0 & 1 & 0  & 1 & 0 \\\hline
$\sigma_{16}$ & 1 & 0 & 1 & 0 & 1 & 0  & 1 & 0 \\\hline
$\sigma_{17}$ & 1 & 0 & 1 & 0 & 1 & 0  & 1 & 0 \\\hline
$\sigma_{18}$ & 1 & 0 & 1 & 0 & 1 & 0  & 1 & 0 \\\hline
$\sigma_{19}$& 0 & 0 & 0 & 0 & 0& 0 & 0 & 0 \\\hline
$\sigma_{20}$ & 0 & 1 & 0 & 0 & 0& 0 & 0 & 0 \\\hline
$\sigma_{21} $& 0 & 2 & 0 & 1 & 0 & 0 & 0 &  0\\\hline
$\sigma_{22} $& 0 & 2 & 0 & 2 & 0 & 2 & 0 & 0 \\\hline
$\sigma_{23}$ & 0 & 1 & 0 & 1 & 0 & 1 & 0 &  1\\\hline
$\sigma_{24} $& 0 & 1 & 0 & 1 & 0 & 1 & 0 &  1\\\hline
$\sigma_{25}$ & 0 & 1 & 0 & 1 & 0 & 1 & 0 &  1\\\hline
$\sigma_{26} $& 0 & 1 & 0 & 1 & 0 & 1 & 0 & 1\\\hline
$\sigma_{27} $& 0 & 1 & 0 & 1 & 0 & 1 & 0 &  1\\\hline
$\sigma_{28} $& 0 & 1 & 0 & 1 & 0 & 1 & 0 &  1\\\hline
$\sigma_{T} $  & 0 & 1 & 0 & 1 & 0 & 1 & 0 & 1 \\\hline
$\sigma_{G}$  & 1 & 0  & 1 & 0 & 1 & 0 & 1 & 0 \\\hline
\end{tabular}
\\
\end{tabular}
\caption{$F_{KB}\downarrow\omega$ {\it (left)} and $F_{KB}\uparrow\omega$ {\it (right)} of the SN P system of Figure \ref{snps_1}}\label{Tab_1}
\end{table}}

\section{Conclusions anf Future Work}\label{Con}
In the last years, the success of technological devices inspired in the connections on neurons in the brain have is doubtless. Almost each day we read news about new achievements obtained by new models or new architectures. Many of the recent developments on neural networks get new knowledge able to predict or classify with an impressive accuracy, but such implicit knowledge is not human readable. Recently, many researchers have started to wonder how {\it translate} this implicit knowledge into a set of rules in order to be understood by humans and then, to be able to introduce new improvements in the technical designs. In the literature, different approaches by using connectionist models for logic-based representation and reasoning can be found. For example, in \cite{DBLP:journals/neco/Pinkas91}, a study of the relation between the SAT problem the minimizing energy in several types of neural networks is presented.

Such translation needs bridges and two of them can be, on the one hand, a set of logic based study there a statement can be considered {\it True} or {\it False} in some sense and them, to be able of apply inference rules to acquire more knowledge and, on the other hand, a neural-inspired model able to handle with binary information, as SN P systems do.   

In this paper, we propose a possible bridge by studying two non-monotonic logic inference rules into a neural-inspired model. This new point of view could shed a new light to further research possibilities. On the one side, to study if new inference rules can be studied in the framework of SN P systems. On the other side, if other bio-inspired models are also capable of dealing with logic inference rules. 

Recently there exist other approaches to model logic-based reasoning with neural models
that tackle questions on entailment and satisfiability. In \cite{AlzaeemiSAKM2017}, the authors use a type of
Hopfield networks to model and solve non-horn 3-SAT, although are models of continuous
nature. Moreover, SN P systems can be useful models to both design and verify logic-based tasks. As future work, an interesting research line can be to discretize classical continuous spiking models and to model them via SN P systems. The target is to explore techniques for verifying and validating such models in industrial applications as robotics \cite{DBLP:journals/finr/BingMRHK18,DBLP:journals/corr/abs-1902-06410}.

\section*{Acknowledgements}
Daniel Rodr\'iguez Chavarr\'ia thanks the partial support by the Youth Employment Operative Program, co-financied with FEDER founds.

\bibliographystyle{unsrt}  
\bibliography{../../../../MiBiblio}

\end{document}